\documentclass{article}
\usepackage[T1]{fontenc}
\usepackage{amsfonts}
\usepackage{bm}
\usepackage{graphicx}
\usepackage{relsize}
\usepackage{subcaption}
\usepackage{amssymb}
\usepackage{hyperref}
\usepackage{amsmath}
\usepackage{dsfont}
\usepackage{amsthm}
\usepackage{booktabs}
\usepackage{rotating}
\usepackage{multirow}
\usepackage{arxiv}

\newtheorem{theorem}{Theorem}

\newtheorem{proposition}{Proposition}

\theoremstyle{definition}
\newtheorem{definition}{Definition}

\newtheorem{property}{Property}

\title{Can we trust our models?\\ Epistemic calibration in second-order classification}
\author{
 Arthur Hoarau\\
 Université de Lorraine, CentraleSupélec\\
 Loria, CNRS,  Metz, France\\
 \texttt{arthur.hoarau@centralesupelec.fr} \\
}

\begin{document}

\maketitle

\begin{abstract}
Uncertainty estimation is critical for deploying machine learning models in high-stakes settings. However, classical calibration only assesses the reliability of predicted probabilities and does not evaluate whether epistemic uncertainty estimates are themselves trustworthy. This limitation is particularly relevant for second-order classification models.
We introduce epistemic calibration, a principled criterion that measures whether reported epistemic uncertainty faithfully reflects the dispersion of model predictions around the ground truth. We show that epistemic calibration is a strictly stronger notion than classical calibration and captures failure modes invisible to standard metrics. We relate this work to the existing literature through an impossibility theorem that holds under the epistemic calibration hypothesis. To operationalize this concept, we propose the Expected Epistemic Calibration Error (EECE), which we prove to be a consistent estimator of a True Epistemic Calibration Error (TECE). Experiments across a broad range of uncertainty quantification methods show that epistemic calibration is a coherent and meaningful criterion and reveal substantial differences across methods, despite similar predictive performance.
\end{abstract}

\section{Introduction}\label{section:intro}

Research in machine learning has long focused on improving model raw performance. However, with the recent breakthroughs in artificial intelligence (vision, recommendation, language models, etc.) and its integration into an ever-increasing number of sometimes critical systems, new challenges are emerging. Growing attention is now being paid to the ecological, economic, and societal costs of AI~\cite{Time2023}, as well as to the demand for more frugal models. For models already capable of achieving near-optimal performance, it becomes crucial to capture and report the uncertainty associated with the most sensitive local predictions, even if they are rare. A single failure can be fatal in safety-critical systems, such as autonomous driving, medical or space applications. The recent revolution of large language models and their democratization have also raised pressing concerns about trust, particularly in the face of the overconfidence that some models display in their predictions. In most applications, so-called ``black-box'' models are implemented, and research partly focuses on making more transparent the predictions of these sometimes very complex models~\cite{niculescu2005,guo2017,kumar2018}.

In this setting, evaluating machine learning models requires going beyond raw predictive performance and explicitly accounting for uncertainty and trust. This motivates a multi-level view of calibration. At level-0, the \emph{accuracy} measures the correctness of the predictions or how often a model makes the correct prediction. At level-1, the \emph{uncertainty} measures the lack of confidence of the model, (i.e. the reliability of its predictions).
At level-2, \emph{trust} measure the confidence that users can place in a model’s predictions, behaviour, and decisions. Each of these levels can be calibrated. Calibration~\cite{niculescu2005, Vaicenavicius2019} refers to the degree to which a model’s outputs, confidence estimates, or assessments are aligned with reality or true outcomes. 

In probabilistic classification, a standard and widely used measure of predictive uncertainty is the Shannon entropy~\cite{Shannon1948} of the conditional predictive distribution. This somewhat simplistic view assumes that model uncertainty increases as the predictive distribution approaches uniformity (i.e., when all classes are equally probable). While consistent, this approach remains incomplete: it does not distinguish between a model that hesitates between two classes because the problem is inherently difficult, and a model that is uncertain simply because it lacks sufficient training.

In the literature, two primary types of uncertainty are typically distinguished~\cite{Hora1996}: aleatoric uncertainty (data uncertainty) and epistemic uncertainty (model uncertainty). Aleatoric uncertainty arises from the stochastic nature of the data-generating process, while epistemic uncertainty is associated with a lack of knowledge. The former is generally considered irreducible, whereas the latter can be mitigated by acquiring additional data. 
To disentangle aleatoric and epistemic uncertainty, numerous methods have been proposed in recent years~\cite{eyke2019}. These methods can be grouped into four main families of models.

\begin{figure}
    \centering
    \begin{subfigure}[c]{0.52\linewidth}
    \includegraphics[width=0.49\linewidth]{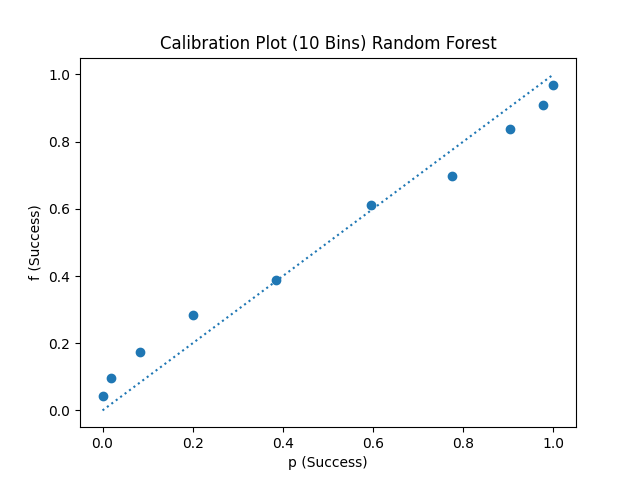}
    \includegraphics[width=0.49\linewidth]{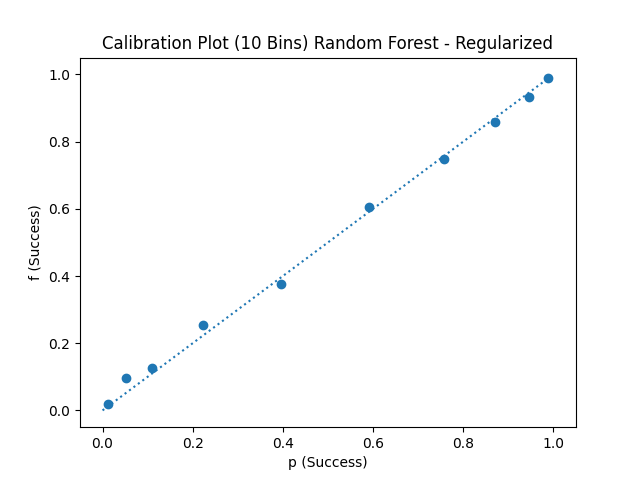}
    \caption{}\label{fig:random_f_cal}
    \end{subfigure}
    \begin{subfigure}[c]{0.19\linewidth}
    \vspace{12pt}
    \includegraphics[width=1.02\linewidth]{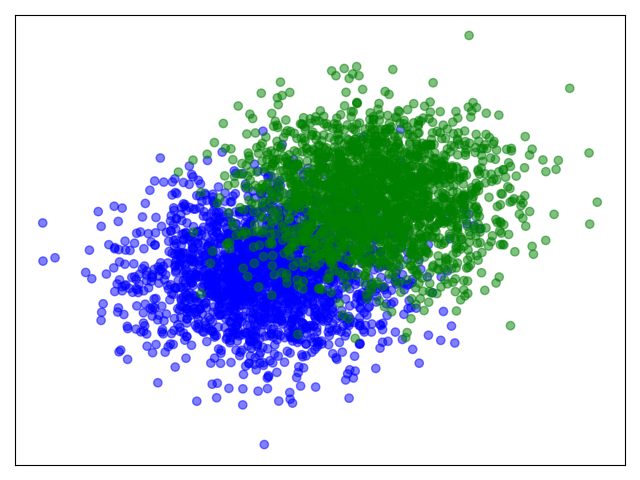}
    \vspace{0pt}
    \caption{}\label{fig:dataset}
    \end{subfigure}
    \begin{subfigure}[c]{0.27\linewidth}
    \includegraphics[width=\linewidth]{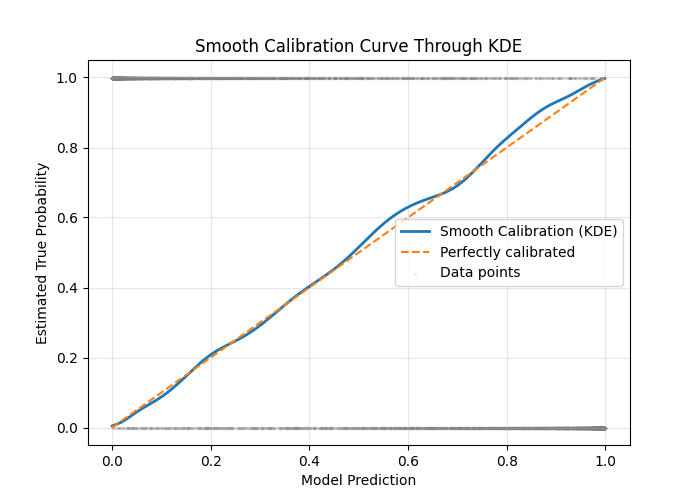}
    \caption{}\label{fig:cal_kde}
    \end{subfigure}
    \caption{\ref{fig:random_f_cal}: Calibration plot with bins \emph{vs.} Random Forest. Regularization on depth or leaf density improves calibration. \ref{fig:dataset}: Toy Dataset used in the following to illustrate calibration and epistemic calibration. Two-class partially overlapping Gaussian kernels. \ref{fig:cal_kde}: Smooth calibration curve through Kernel Density Estimation for an almost perfectly calibrated Bayesian Logistic Regression.}
\end{figure}

Bayesian methods compute a posterior distribution over a model’s parameters and construct a second-order distribution over class probabilities via Bayesian model averaging~\cite{kendall2017,Depeweg2018,valiuddin2025a} 
\begin{equation}
	p(h \mid \mathcal{D}) \propto p(h) p(\mathcal{D} \mid h).
\end{equation}
Uncertainty quantification in machine learning is traditionally grounded in Bayesian learning, where $p(h \mid \mathcal{D})$ denotes the posterior probability of a specific model $h \in \mathcal{H}$ given a dataset $\mathcal{D}$. Intuitively, $p(\cdot \mid \mathcal{D})$ reflects the learner’s state of knowledge, and thus its epistemic uncertainty: the more peaked this distribution is (i.e., one model is much more likely than the others), the less epistemically uncertain the learner is about $h^*$ being the optimal model. Crucially, $h^*$ being the best model does not imply that its prediction $\hat{y} = h^*(x)$ or $p(\hat{y} \mid x)$ is certain from an aleatoric perspective. In other words, a model can be epistemically confident that the task itself is intrinsically aleatorically difficult.
Despite its strong theoretical foundations, Bayesian learning has also faced substantial criticism~\cite{fellaji2024}, particularly due to its restrictive assumptions and the reliance on empirical approximations needed to make inference tractable.

Ensemble methods capture epistemic uncertainty through ensemble diversity. This can be achieved by randomizing the training data via bootstrapping, varying neural network architectures using techniques such as MC Dropout, or randomizing the optimizer, as in Deep Ensembles~\cite{Lakshminarayanan2017, Mobiny2021}. Epistemic uncertainty is high when the disagreement among estimators is large, whereas aleatoric uncertainty is high when the mean entropy of the estimators is high. It has been shown that some of these techniques can be interpreted as approximations of the second-order posterior distribution discussed previously~\cite{kendall2017}. 
Although intuitive and widely applicable, this family of methods has also been criticized regarding the nature of the uncertainty it captures~\cite{jimenez2026}.

Evidential deep learning adopts a frequentist perspective and aims to minimize both model error and epistemic uncertainty, which is represented by a second-order probability, without requiring the specification of a prior or posterior over the parameters. For multi-class classification problems, the Dirichlet distribution is typically employed~\cite{Sensoy2018, Ulmer2023}. 
Despite achieving reasonable performance, this recent family of methods has faced several criticisms~\cite{Bengs2023}. In particular, 
some expected behaviours have been shown to be impossible to achieve~\cite{jurgens2024}, leading to the general conclusion that evidential deep learning is unsuitable for proper epistemic uncertainty quantification.

Density-based and distance-based methods typically model aleatoric and epistemic uncertainty directly over the variable space~\cite{Nguyen2022, Hoarau2024-ml}. However, these methods can also construct a second-order distribution in a post-processing step~\cite{Charpentier2020}. Typically, epistemic uncertainty associated with a prediction is high when its likelihood is low for all classes in the dataset, whereas aleatoric uncertainty is high when the instance is likely to belong to multiple classes. The main drawback of such methods is that epistemic uncertainty is unbounded or not theoretically justified, although they can at least produce a coherent ranking of predictions.

Because of the intrinsic limitations of classical probability theory in capturing different forms of ignorance and imprecision, alternative mathematical frameworks have been developed~\cite{GUO2024}. Belief function theory, for instance, can represent multiple levels of ignorance. Imprecise probabilities enable the construction of credal sets, formalizing lower and upper bounds for each probability. Likewise, fuzzy-set theory provides a way to represent fuzzy class boundaries.

Despite the rich body of work devoted to uncertainty quantification and the diversity of existing modelling approaches, comparatively little attention has been paid to how well epistemic uncertainty estimates are calibrated. In particular, while many methods aim to represent second-order uncertainty, there is a lack of principled tools to assess whether these representations faithfully reflect the true degree of epistemic uncertainty and can therefore be trusted. Addressing this gap is the focus of this paper:
\begin{itemize}
    \item We introduce a novel notion of \emph{epistemic calibration} to ensure the trustworthiness of second-order classification models.
    \item We relate our work to the existing literature through an impossibility theorem under epistemic calibration hypothesis.
    \item We propose the \emph{Expected Epistemic Calibration Error} as a metric to quantify the degree of epistemic calibration of a given model.
    \item We prove that the \emph{Expected Epistemic Calibration Error} is a consistent estimator of the \emph{True Epistemic Calibration Error}.
    \item We identify and analyse some models that exhibit particularly strong or weak epistemic calibration.
\end{itemize}
The paper is structured as follows. Section~\ref{section:background} reviews classical calibration, the expected calibration error, and the quantification of epistemic uncertainty. Section~\ref{section:ep_cal} introduces our epistemic calibration framework in relation to an existing literature, along with the Expected Epistemic Calibration Error. Section~\ref{section:xp} presents a series of experiments, ranging from expected behaviour analyses to a broader benchmark of existing methods. Finally, Section~\ref{section:conclusion} concludes the article.

\section{Calibration \& Epistemic Uncertainty}\label{section:background}

In this section we first recall the definition of calibration and the expected calibration error. We then review the representation and quantification of epistemic uncertainty in a machine learning setting.

\subsection{Calibration}

Let $(X,Y)$ be a pair of random variables (we omit the notion of random vectors for the sake of simplicity), where $X$ takes values in the input space 
$\mathcal{X} \subseteq \mathbb{R}^p$ and $Y$ takes values in the label set 
$\mathcal{Y} = \{0,1\}$. Denote the joint distribution by $P_{X,Y}$. 
We define a dataset $\mathcal{D}$ of $N$ samples
\begin{equation}
\mathcal{D} := \{(x_n, y_n)\}_{n=1}^N,
\end{equation}
where $(x_n, y_n)$ are i.i.d draws from $P_{X,Y}$. The ground truth conditional class-probability is defined by
\begin{equation}
P(Y = y \mid X = x).
\end{equation}

Let $h : \mathcal{X} \to [0,1]$ be a predictive model in the hypothesis space $h\in\mathcal{H}$ where $h(x)$ represents 
an estimate of the probability $P(Y=1 \mid X = x)$. 

\begin{definition}[Calibration]
We say that $h$ is \emph{calibrated} iff, for any $\hat{p} \in [0,1]$, we have
\begin{equation}
P(Y=1 \mid h(X)= \hat{p}) = \hat{p}.
\end{equation}
\end{definition}
Rephrased, the predictive distribution of the classifier should reliably represent its uncertainty. In the general case of multi-class classification, different degrees of calibration can be distinguished~\cite{brocker2009, guo2017, Vaicenavicius2019}. However, we do not consider such distinctions in this paper, as these measures collapse into the one presented above in our binary classification setup. Another question eluded to in this paper and considered out of scope is the recalibration task, which involves building a calibrated classifier either \emph{ad hoc} during training~\cite{kendall2017, Lakshminarayanan2017},  or \emph{post hoc} as a recalibration of an already trained classifier~\cite{guo2017, zadrozny2001, Vaicenavicius2019}.

A common way to estimate $P(Y=1 \mid h(X)= \hat{p})$ is to use binning~\cite{zadrozny2001,naeini2015}, as presented for Random Forests in Figure~\ref{fig:random_f_cal} for a given dataset (cf. Figure~\ref{fig:dataset}). A model is considered perfectly calibrated if the predicted probability of success matches the true frequency in each bin (represented by a point in the figure). This figure highlights that a regularized Random Forest is better calibrated. A regularized tree (through \emph{pre-} or \emph{post-}pruning) is shallower, but most importantly its leaves are more dense, which naturally better matches the true frequency in each bin.
Another body of the literature~\cite{popordanoska2022, jurgens2025} advocates for a continuous spectrum instead of the discrete binning strategy. One can use smooth density estimations to obtain differentiable functions. Such an example is presented for illustrative purposes in Figure~\ref{fig:cal_kde} using a Kernel Density Estimation (with a Beta Kernel and a bandwidth of 0.2) with a Bayesian Logistic Regression.

\subsubsection{Expected Calibration Error}

The Expected Calibration Error (ECE) quantifies the difference between predicted probabilities and empirical frequencies.
Let $\{I_m\}_{m=1}^M$ be a partition of the interval $[0,1]$ into $M$ disjoint bins. For each bin $I_m$, let define the index set
\begin{equation}
B_m := \{ n \in \{1,\dots,N\} \mid h(x_n) \in I_m \}.
\end{equation}
Let the empirical accuracy and the confidence within bin $I_m$ be
\begin{equation}
\mathrm{acc}(I_m) := \frac{1}{|B_m|}\sum_{n \in B_m}\mathds{1}_{\{y_n = 1\}}, \quad \mathrm{conf}_h(I_m) := \frac{1}{|B_m|}\sum_{n \in B_m} h(x_n).
\end{equation}

\begin{definition}[Expected Calibration Error]
The \emph{ECE} is defined as
\begin{equation}
\mathrm{ECE}_M := \sum_{m=1}^M \frac{|B_m|}{N}\left| \mathrm{acc}(I_m) - \mathrm{conf}_h(I_m) \right|.
\end{equation}
\end{definition}
It approximates the true calibration error by partitioning the predicted probabilities into $M$ bins and computing a weighted average of the discrepancy between empirical accuracy and mean predicted confidence within each bin. The ECE provides a discrete approximation of the True Calibration Error.
\begin{definition}[True Calibration Error]
The \emph{TCE} is defined as
\begin{equation}
\mathrm{TCE} :=\mathbb{E}_X \left[\left|P(Y=1 \mid h(X)) - h(X)\right|\right],
\end{equation}
which vanishes if and only if the model is perfectly calibrated. 
\end{definition}

In practice, the ECE estimates the TCE by discretizing the predicted probability interval. 
Although this discretization introduces a small bias, the ECE is a consistent estimator of the TCE, converging to the true value as the number of bins and samples increase.
\begin{theorem}\label{theorem:original}
    Let $r(\hat{p}) := P(Y=1\mid h(X) = \hat{p})$ be a Lipschitz continuous calibration function and $\{I_m\}^M_{m=1}$ be a partition of the interval $[0,1]$ such that $\underset{M\rightarrow \infty}{\lim} \underset{\;1 < m < M}{\max} \underset{\;a, b \in I_m}{\sup} ||a - b||_2 = 0$. Then
    \begin{equation}
        \underset{M\rightarrow \infty}{\lim}\underset{N\rightarrow \infty}{\lim} \mathrm{ECE}_M = \mathrm{TCE},
    \end{equation}
    with limits in the almost sure sense.
\end{theorem}
Furthermore, it has been shown by~\cite{Vaicenavicius2019} that ECE is a consistent estimator of TCE for any function (other than the $L_1$ norm used in this paper) continuous and uniformly continuous in its first argument.

\subsection{Epistemic uncertainty}\label{section:ep_unc}

Let $\Pi$ be a second-order model $\Pi \in \mathcal{P}(\mathcal{H})$ and $h\sim\Pi$ be a hypothesis sampled from $\Pi$. Depending on the context, $\Pi$ may represent a Bayesian posterior over models (\emph{Bayesian Learning}), an ensemble sampling distribution (\emph{Ensemble Learning}), a credal selection rule (\emph{Credal Sets}), or another second-order uncertainty representation (\emph{Evidential Deep Learning}, \emph{Dempster-Shafer Theory}). In this paper, we take the second-order distribution $\Pi$ as given, and our theoretical analysis is conducted directly on this object. We therefore do not commit to any particular method for constructing $\Pi$, as the framework applies broadly across the paradigms described above. Practical considerations regarding how such distributions can be obtained in practice are discussed in the experimental section.

For any input $x \in \mathcal{X}$, each hypothesis $h \in \mathcal{H}$ produces 
a predicted posterior probability $h(x)$ estimating $ P(Y=1\mid X =x)$ which is averaged into the prediction of the model $\Delta(x) := \mathbb{E}_{h\sim\Pi}[h(x)]$. 

\begin{definition}[Variance-based estimation]
The estimated \emph{epistemic uncertainty} at $x$ is defined as the variance of the posterior predictive distribution
\begin{equation}
EU(x) := \mathrm{Var}_{h\sim\Pi}\big[h(x)\big].
\end{equation}
\end{definition}

This definition does not require a parametrized hypothesis space, as in Bayesian learning~\cite{kendall2017}. It instead accommodates more general representations of epistemic uncertainty, such as credal sets, ensembles of classifiers, and belief-based representations.
Another classical representation of epistemic uncertainty, intrinsic to the Bayesian family of methods, is the information-theoretic decomposition. Let $\mathcal{D}$ denote the training dataset, $\theta$ the model parameters, $p(y|x,\theta)$ the predictive distribution given parameters $\theta$, $p(\theta|\mathcal{D})$ for the posterior over parameters given the data, and $p(y|x,\mathcal{D}) := \mathbb{E}_{p(\theta|\mathcal{D})}[p(y|x,\theta)]$ for the posterior predictive distribution. The model uncertainty can be decomposed as follows
\begin{equation}
    \underbrace{H[p(y|x,\mathcal{D})]}_{\text{Total Uncertainty}} = \;\;\underbrace{\mathbb{E}_{p(\theta|\mathcal{D})}[H\left[p(y|{x},\theta)]\right]}_{\text{Aleatoric unc.}}\;\;\; + \underbrace{I(y,\theta|{x},\mathcal{D})}_{\text{Epistemic unc.}},
\end{equation}
where $H(\cdot)$ and $I(\cdot\mid\cdot)$ respectively denote Shannon's entropy and mutual information. This decomposition is not discussed further in this paper. Instead, we adopt the broader definition based on the variance of model predictions.

While classical calibration ensures that the mean prediction is reliable, 
we also want the uncertainty estimate to be reliable. We thus introduce in the following section the notion of epistemic calibration.

\section{Epistemic Calibration}\label{section:ep_cal}

Classical calibration ensures that each single prediction is calibrated on average. This notion can be extended to second-order models by considering mean calibration. However, a mean-calibrated second-order model may fail to calibrate predictions individually: two predictions sharing the same mean but differing in variance may not be calibrated on average individually. In this section, we formally introduce the notion of \emph{epistemic calibration}, under which each individual prediction is calibrated.
\begin{definition}
A second-order predictor $\Pi$ is epistemically calibrated iff, $\forall\, \Psi \in \mathcal{P}([0, 1])$
\begin{equation}\label{eq:ep-cal}
    P(Y=1\mid \Pi(X) = \Psi) = \Delta_\Psi,
\end{equation}
where $\Delta_\Psi = \mathbb{E}_{\phi\sim\Psi}[\phi]$ is the mean prediction of the fixed distribution.
\end{definition}
Rather than proceeding directly through the definition, this section builds around equation~\eqref{eq:ep-cal} to progressively unveil the components and properties of epistemic calibration. We first relate epistemic calibration to the existing literature through an impossibility theorem. We then introduce the True Epistemic Calibration Error and the Expected Epistemic Calibration Error as a consistent estimator thereof. Further experiments show that epistemic calibration is a natural property that some second-order models tend to satisfy.

\subsection{An impossibility theorem}

Recent literature has raised concerns about the disentanglement between aleatoric and epistemic uncertainty~\cite{kirchhof2025, smith2025}. In particular, expected behaviours may not hold when increasing the number of samples or the complexity of the model~\cite{fellaji2024}, and strong correlations between epistemic and aleatoric components have been empirically observed~\cite{mucsanyi2024}.
We focus here on a related line of work~\cite{jimenez2026}, where a bias-variance decomposition of epistemic uncertainty has been proposed. For simplicity, let $p_x:= P(Y=1 \mid X =x)$ denote the true conditional distribution and $\hat{p}_{N,\gamma,x} := h_{N,\gamma}(x)$ the predictive distribution where the dataset $\mathcal{D}$ of size $N$ and $\gamma$ are random variables respectively reflecting the data uncertainty and the procedural uncertainty. Given $\ell : \mathcal{P}([0,1]) \times [0,1] \rightarrow \mathbb{R}$ a strictly proper scoring rule inducing a divergence $D$, the authors in~\cite{jimenez2026} decompose the epistemic uncertainty into 
\begin{equation}
    \underbrace{D_{\ell}(p_x, \bar{p}_x)}_{\text{generalized bias}} + \quad \underbrace{\mathbb{E}_{N,\gamma}[D_{\ell}(\bar{p}_x, \hat{p}_{N,\gamma, x})]}_{\text{generalized variance}},
\end{equation}
where 
\begin{equation}
    \bar{p}_x := \underset{p_x\in\mathcal{P}([0,1])}{\arg \min}\; \mathbb{E}_{N,\gamma}[D_{\ell}(p_x, \hat{p}_{N,\gamma,x})],
\end{equation}
is the Bregman centroid refereed to as the \emph{dual mean}.
In this work the authors argue that ``recent and highly-cited epistemic uncertainty estimation methods are fundamentally incomplete'', a claim they attribute to the aforementioned bias term, which proves intractable in practice.
Therefore, we formulate the following impossibility theorem stating that an epistemically calibrated second-order model cannot encode bias in its dispersion (i.e. the epistemic uncertainty estimation).
\begin{theorem}\label{theorem:impossibility}
    For an epistemically calibrated predictor $\Pi$, the spread of the distribution carries no additional information about $Y$ beyond what the mean already encodes.
\end{theorem}
Proof for Theorem~\ref{theorem:impossibility} is provided in Appendix~\ref{app:proofs}, and the following properties can be derived.
\begin{property}
    If a second-order predictor $\Pi$ is epistemically calibrated, then the two following property hold
    \begin{itemize}
        \item Mean calibration: $P(Y=1\mid \Delta(X) = \hat{p}) = \hat{p}$, 
        \item Epistemic independence: $P(Y=1\mid \Delta(X) = \Delta_\Psi, EU(X) = EU_\Psi) = P(Y=1\mid \Delta(X) = \Delta_\Psi)$.
    \end{itemize}
\end{property}
Mean calibration shows that epistemic calibration is a stronger notion than classical calibration while epistemic independence is the property that guarantees that the variance of the posterior cannot encode information about $Y$ beyond what is already encoded in the mean (i.e., $Y \perp EU(X) \mid\Delta(X)$). An important clarification is that in the cited work, the authors refer to bias as a divergence from the ground truth probability $P(Y = 1 \mid X = x)$, whereas in this paper we adopt the widely used convention of averaging around fixed model predictions \[P(Y=1\mid h(X) =\hat{p}) = \mathbb{E}[P(Y=1\mid X)\mid h(X) = \hat{p}],\] in order to make the properties derived in the following sections tractable.
In fact, Theorem~\ref{theorem:impossibility} is not contradictory but rather complementary to their work. The authors claim that epistemic uncertainty \emph{globally cannot} be estimated due to generalized bias, whereas we claim that second-order models \emph{locally can} be trusted under epistemic calibration. Furthermore, we introduce in the following section a procedure to test for this property.

\subsection{The Expected Epistemic Calibration Error}

In this section, we reformulate epistemic calibration so as to make the variance of the second-order distribution appear explicitly. Epistemic calibration then becomes a property of the epistemic uncertainty estimate (i.e., the variance of the posterior), which should exactly match the expected squared error of the first-order drawn predictions.
\begin{proposition}\label{proposition:ep-cal}
A second-order predictor $\Pi$ is epistemically calibrated iff, $\forall\, \Psi \in\mathcal{P}([0, 1])$
\begin{equation}\label{eq:ep-cal2}
    \mathbb{E}_{\phi \sim \Psi}\!\left[\left(P\!\left(Y=1 \mid \Pi(X) = \Psi\right) - \phi\right)^2\right] = EU_\Psi,
\end{equation}
where $EU_\Psi = \mathrm{Var}_{\phi\sim\Psi}[\phi]$ is the variance of the fixed distribution.
\end{proposition}
The proof of Proposition~\ref{proposition:ep-cal} is provided in Appendix~\ref{app:proofs}.  Intuitively, among regions where the model reports an epistemic uncertainty of $EU_\Psi$, the actual variance of individual model predictions around the true frequency should equal that value.

With regard to the relationship with classical calibration. Epistemic calibration is a \emph{stronger} notion. Classical calibration only concerns the aggregated prediction $\Delta(x)$ while epistemic calibration additionally requires that the announced uncertainty $EU(x)$ correctly reflects the actual dispersion of models around the truth. A model can be classically calibrated while being epistemically miscalibrated if models disagree in a way that doesn't reflect their collective error.

Such as for standard calibration, epistemic calibration can be measured through binning.
Let $H = \{h_1, \dots, h_{|H|}\}$ be a finite set of classifiers and  $\{I_m\}_{m=1}^M$ be a partition of the joint prediction space $[0,1]^{|H|}$. This partition can be constructed via a Cartesian product of intervals (grid-based) or through data-dependent clustering (e.g., K-Means) to avoid empty bins.
Depending on the method, the total number of bins can grow exponentially with the ensemble size $|H|$, so in practice most bins $B_{m}$ may be empty unless $N \gg M$. This curse of dimensionality is an inherent limitation of the estimator. An additional study on the impact of the number of bins is provided in Appendix~\ref{app:xp}, demonstrating that satisfactory results can be achieved with a small number of bins. For each bin $I_{m}$, let define the index set
\begin{equation}
    B_{m} := \{n \in \{1, ..., N\}\mid (h(x_n))_{h\in H} \in I_{m} \}.
\end{equation}
Let the empirical accuracy, the confidence and the trust within bin $I_m$ be
\begin{equation}
\mathrm{acc}(I_m) := \frac{1}{|B_m|}\sum_{n \in B_m}y_n,\quad
\mathrm{conf}_h(I_m) := \frac{1}{|B_m|}\sum_{n \in B_m} h(x_n), \quad
    \mathrm{trust}(I_m) := \frac{1}{|B_m|}\sum_{n \in B_m} EU(x_n).
\end{equation}
The Expected Epistemic Calibration Error (EECE) quantifies the difference between predicted epistemic uncertainties and empirical expected quadratic errors.
\begin{definition}[Expected Epistemic Calibration Error]
The \emph{EECE} is defined as
\begin{equation}\label{eq:eece}
\mathrm{EECE}_{M,H} := \sum_{m=1}^M\frac{|B_m|}{N}\left|\left(\frac{1}{|H|}\sum_{h\in H}\left(\mathrm{acc}(I_m) - \mathrm{conf}_h(I_m)\right)^2\right) - \mathrm{trust}(I_m)\right|.
\end{equation}
\end{definition}
The EECE approximates the true epistemic calibration error by partitioning the joint prediction space into $M$ bins and computing a weighted average of the discrepancy between the empirical expected squared error and the mean reported epistemic uncertainty within each bin. It provides a discrete approximation of the true epistemic calibration error.
\begin{definition}[True Epistemic Calibration Error]
The \emph{TECE} is defined as
\begin{equation}
\mathrm{TECE} := \mathbb{E}_X\left[\left|\mathbb{E}_{h\sim\Pi}\!\left[\left(P\!\left(Y=1 \mid \Pi(X) \right) - h(X)\right)^2\right] - EU(X)\right|\right],
\end{equation}
which vanishes if and only if the model is perfectly epistemically calibrated.
\end{definition}
Similarly to the ECE and the TCE, the EECE estimates the TECE by discretizing a product space of joint probability interval. Although this discretization also introduces a bias, we show that the EECE is a consistent estimator of the TECE, converging to the true value as the number of bins and samples increase.
\begin{theorem}\label{theorem:convergence}
    Let $r(\Psi) := P(Y=1\mid \Pi(X) = \Psi) $ and $\nu(\Psi) := \mathbb{E}_{\phi \sim \Psi}\!\left[\left(P\!\left(Y=1 \mid \Pi(X) = \Psi\right) - \phi\right)^2\right]$ be two Lipschitz continuous calibration functions and $\{I_m\}_{m=1}^M$ be a partition of the joint prediction space $[0,1]^{|H|}$ into convex bins with metric $d$ such that $\underset{M\rightarrow \infty}{\lim} \underset{\;1 < m < M}{\max} \underset{\;a, b \in I_m}{\sup} d(a,b) = 0$. Then
    \begin{equation}
        \underset{M\rightarrow \infty}{\lim}\underset{|H|\rightarrow \infty}{\lim}\underset{N\rightarrow \infty}{\lim}  \mathrm{EECE}_{M,H} = \mathrm{TECE},
    \end{equation}
    with limits in the almost sure sense.
\end{theorem}
The proof for Theorem~\ref{theorem:convergence} is provided in Appendix~\ref{app:proofs}. With the EECE, we provide an empirical tool for evaluating epistemic calibration. This is further demonstrated in the following section, where we empirically evaluate epistemic calibration and show that, beyond being a natural property, models tend to differ substantially in their epistemic calibration despite sharing similar predictive performance.

\section{Experiments}\label{section:xp}

\begin{figure}
    \centering
    \begin{subfigure}{0.46\linewidth}
    \includegraphics[width=0.48\linewidth]{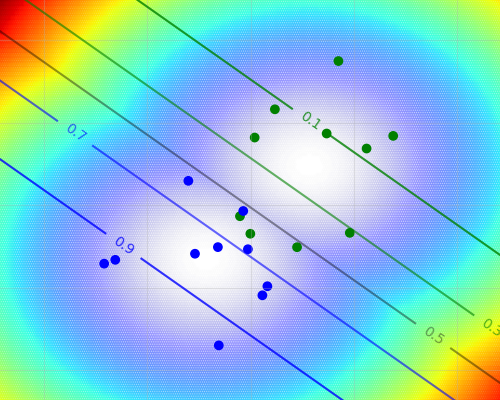}
    \hfill\includegraphics[width=0.48\linewidth]{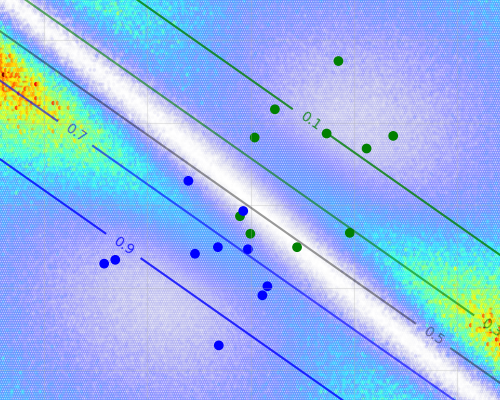}
    \caption{Centroid-based}\label{fig:recal_centroids}
    \end{subfigure}
    \hfill
    \begin{subfigure}{0.46\linewidth}
    \includegraphics[width=0.48\linewidth]{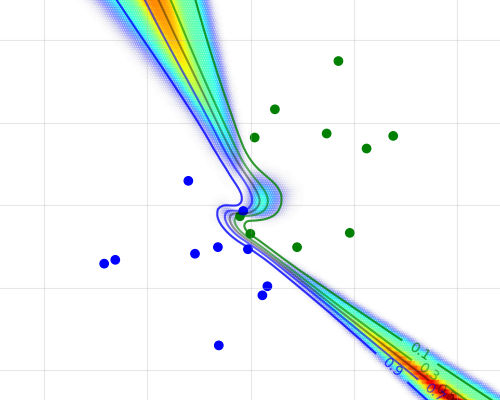}
    \hfill\includegraphics[width=0.48\linewidth]{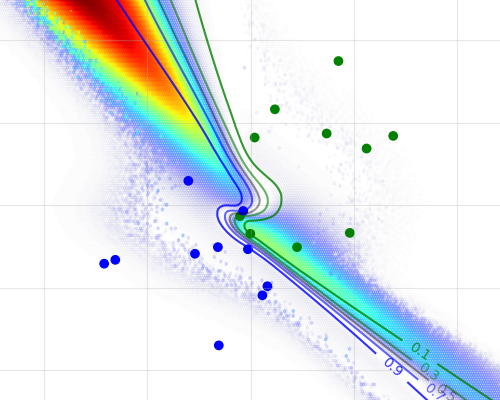}
    \caption{Deep-Ensemble}\label{fig:recal_deepe}
    \end{subfigure}
    \caption{For each model, the left panel displays the estimated epistemic uncertainty $EU(X)$ and the right panel  displays the expected squared error $\mathbb{E}_{h\sim\Pi}\!\left[\left(P\!\left(Y=1 \mid \Pi(X) \right) - h(X)\right)^2\right]$. Red regions for higher values.}
\end{figure}

In this section, we empirically evaluate epistemic calibration through multiple experiments. We first consider a Toy dataset in order to maintain control over the ground truth $P(Y=1 \mid X = x)$ and explore the expected behavior of the EECE. The reader may note that on such a simple dataset, none of the Bayesian or ensemble-based methods are \emph{poorly} epistemically calibrated. We therefore present in Section~\ref{exp:best} a more challenging scenario involving deep vision classification, where we compare epistemic calibration across a range of models.

We consider two groups of methods for estimating epistemic uncertainty. The first group is theoretically compatible with our framework, as their uncertainty estimates are grounded in the variance of a posterior distribution. The second group relies on different justifications and is included solely for illustrative purposes. These methods are not designed to produce epistemically calibrated outputs in our sense, and are therefore expected to perform poorly under our criterion. Further details are available in Appendix~\ref{app:xp}.

\paragraph{Active methods.} As for the studied methods we consider a Bayesian Logistic Regression~\cite{polson2013} with weights and biases following a prior normal distribution of parameters $\mathcal{N}(0, 0.5)$. We also consider Bayesian Deep Learning via Laplace Approximation~\cite{daxberger2022}. For both methods, individual models are sampled from the posterior predictive second-order distribution, and the EECE is computed following Equation~\eqref{eq:eece}. We further include Deep Ensembles~\cite{Lakshminarayanan2017}, Random Forests~\cite{Breiman2001}, DropConnect~\cite{Mobiny2021}, and MC-Dropout~\cite{gal2016}. For these ensemble-based methods, individual ensemble members are treated as samples from the second-order predictive distribution.

\paragraph{Illustrative methods.} For illustrative purposes only, we consider a centroid-based approach~\cite{van-amersfoort20a}, a likelihood based approach~\cite{Nguyen2022}, and a nearest-neighbors based approach~\cite{denoeux1995}. Since these methods do not rely on second-order distributions or ensemble decompositions, we sample hypotheses from a normal distribution using the predicted epistemic uncertainty as the variance, i.e. $\phi\sim\mathcal{N}(h(x), \sqrt{EU(x)})$. This sampling scheme does not reflect the intended semantics of these methods, but allows us to highlight that such a mismatch induces miscalibration with respect to our definition.

\subsection{Evaluating epistemic calibration}

We first consider the Toy dataset in Figure~\ref{fig:dataset} for a completely controlled experiment, where $P_X$ and $P_{Y\mid X}$ are known.
The data are generated according to a mixture of Gaussians:
\begin{equation}
    P(X) = \sum^K_{k=1}\pi_k\cdot\mathcal{N}(X\mid\mu_k, \Sigma_k),
\end{equation}
where $\mathcal{N}(x\mid\mu_k,\Sigma_k)$ denotes the density of a multivariate Gaussian distribution and $\pi_k$ is the \emph{prior} associated with class $k$.

We begin with a voluntarily epistemically miscalibrated model.  Figure~\ref{fig:recal_centroids} presents the epistemic uncertainty of a centroid-based approach and a ``re-calibrated'' version of what should be it's epistemic uncertainty. The epistemic uncertainty is high when test points are far from the 2 class centroids (the four corners of the left panel). However, in reality the model is expected to have a high epistemic uncertainty where the true miscalibrated risk is higher, which is not uniformly the case. In other words, the model says to be less trustworthy in the four corners but fails to detect this properly.

Then, we consider a massively and practically used deep learning model. Figure~\ref{fig:recal_deepe}, presents Deep Ensemble and it's epistemic uncertainty according to a variance-based decomposition. Intuitively, the epistemic uncertainty is high in a ``coherent region'' but in reality the aleatoric calibration error of the individual models is expected to be higher in slightly different regions.

Two additional results are presented in Appendix~\ref{app:xp}: an almost perfectly epistemically calibrated model (Bayesian Logistic Regression) in Figure~\ref{fig:recal_bayesianLR}, and a poorly epistemically calibrated model (Random Forest) in Figure~\ref{fig:recal_forest}. Indeed, Bayesian Logistic Regression is particularly well-suited for two-class Gaussian data when the optimal decision boundary is linear. Additionally, Figure~\ref{fig:ECE_EECE} summarize the epistemic calibration for each model, outlining well calibrated Bayesian and ensemble-based methods, except for Random Forest.  This further supports the claim that a properly specified model should be epistemically calibrated. The following experiment further corroborates that epistemic calibration is a natural property.

\subsection{Expected behavior of the EECE}

This section present an empirical validation of EECE's behavior on a simple controlled dataset. 

Figure~\ref{fig:expected} presents the EECE of Random Forest on our Toy dataset when increasing noise and training set size. The left panel clearly shows that the epistemic calibration error increases as the noise in the training set increases, from almost a factor of 10 (from 0\% noise to 0.95\%). The right panel shows that the EECE decreases when increasing the size of the training set until it reaches a plateau (from .07 with 10 training instances, to .007 with more than 50 instances).

In this experiment, we illustrated the case of Random Forest as it is the most poorly epistemically calibrated model among those compatible with the variance-based framework. The reader may have noticed that on such a simple dataset, most models (excluding the illustrative ones) are quite well epistemically calibrated, and may wonder whether this holds on more practical datasets. The following experiment is therefore conducted on a larger and more complex dataset to address this question.

\begin{figure}
    \centering
    \begin{subfigure}{0.35\linewidth}
    \hfill\includegraphics[width=\linewidth]{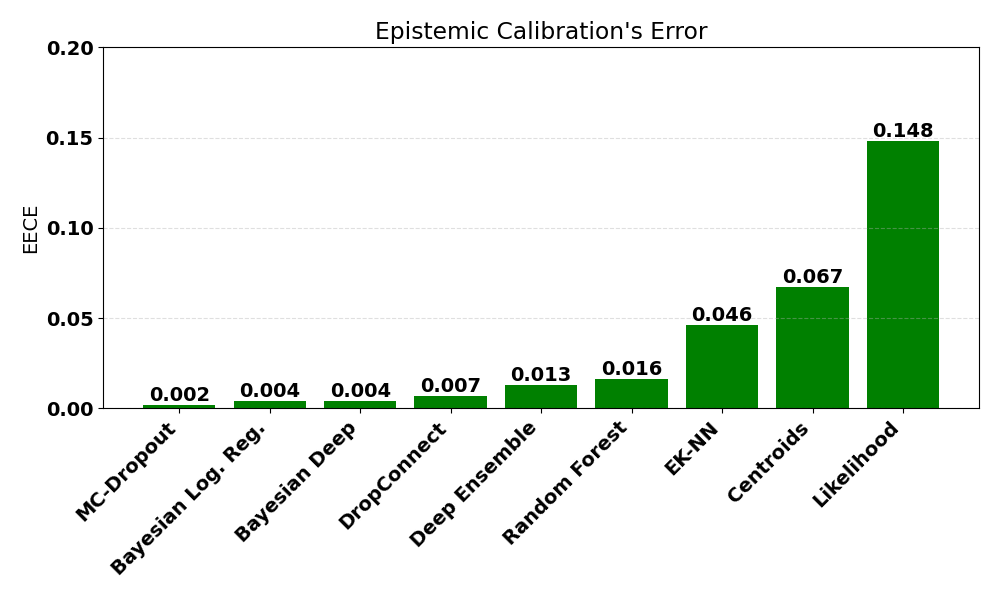}
    \caption{}\label{fig:ECE_EECE}
    \end{subfigure}
    \hfill
    \begin{subfigure}{0.63\linewidth}
    \includegraphics[width=0.49\linewidth]{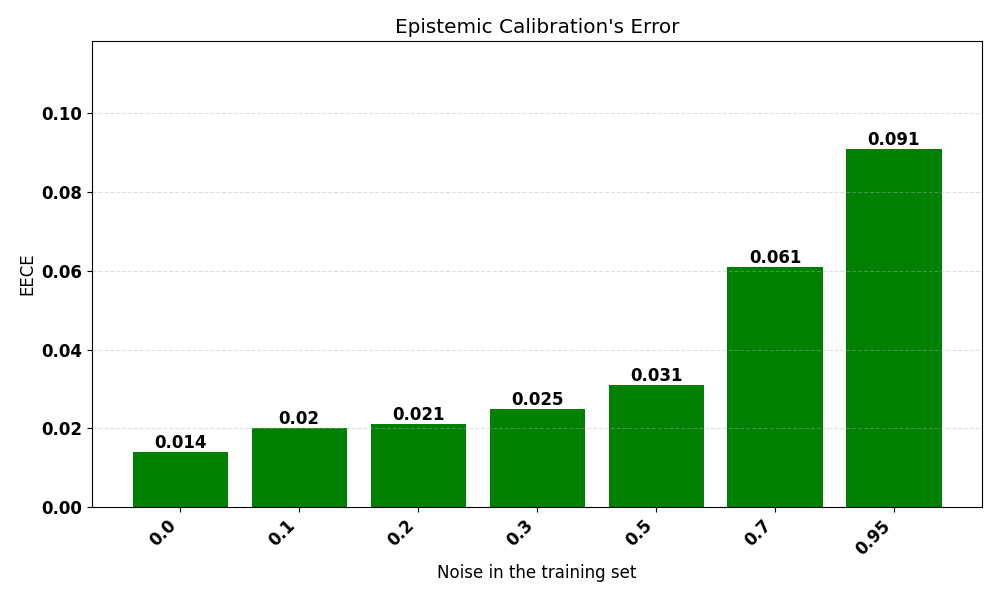}
    \hfill\includegraphics[width=0.49\linewidth]{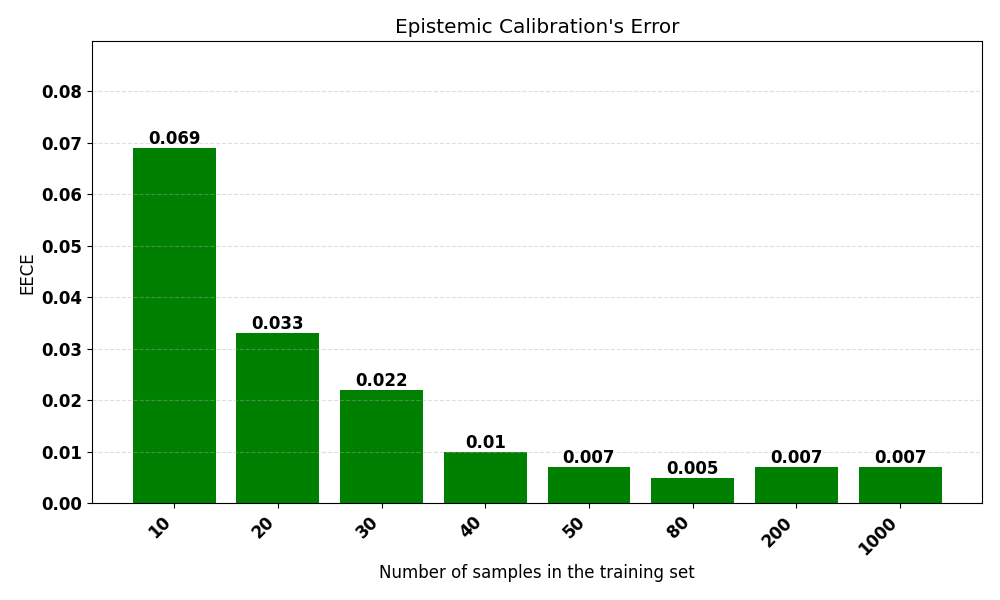}
    \vspace{0.3cm}
    \caption{}\label{fig:expected}
    \end{subfigure}
    \caption{\ref{fig:ECE_EECE}: EECE for different models \emph{vs.} Toy dataset. \ref{fig:expected}: Random Forest's EECE behavior \emph{vs.} increasing noise (left) and training set size (rigt) or Toy dataset.}
\end{figure}

\begin{figure}
    \centering
    \begin{subfigure}{0.45\linewidth}
    \hfill\includegraphics[width=\linewidth]{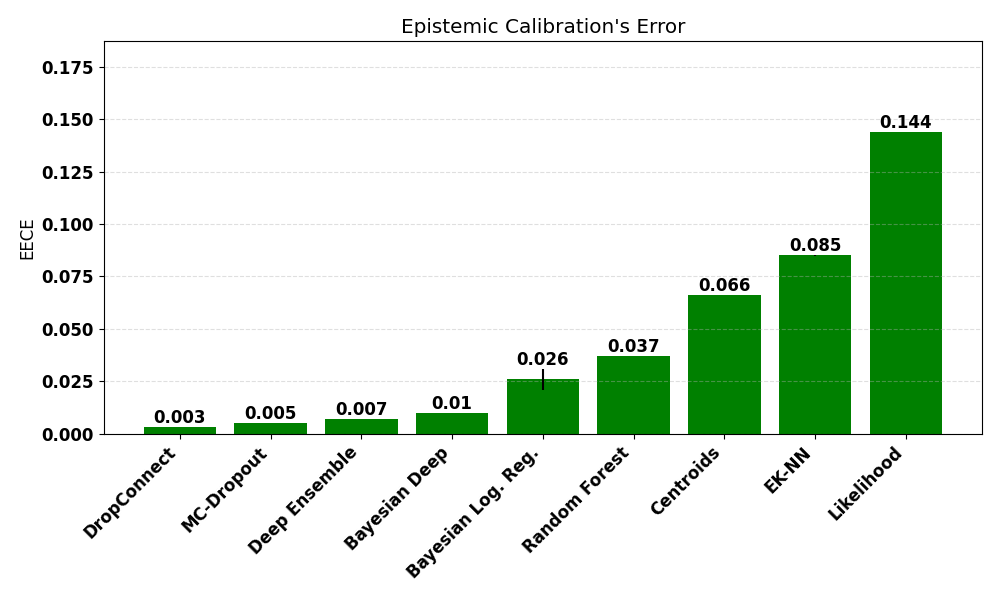}
    \caption{}\label{fig:bench}
    \end{subfigure}
    \hfill
    \begin{subfigure}{0.54\linewidth}
    \includegraphics[width=\linewidth]{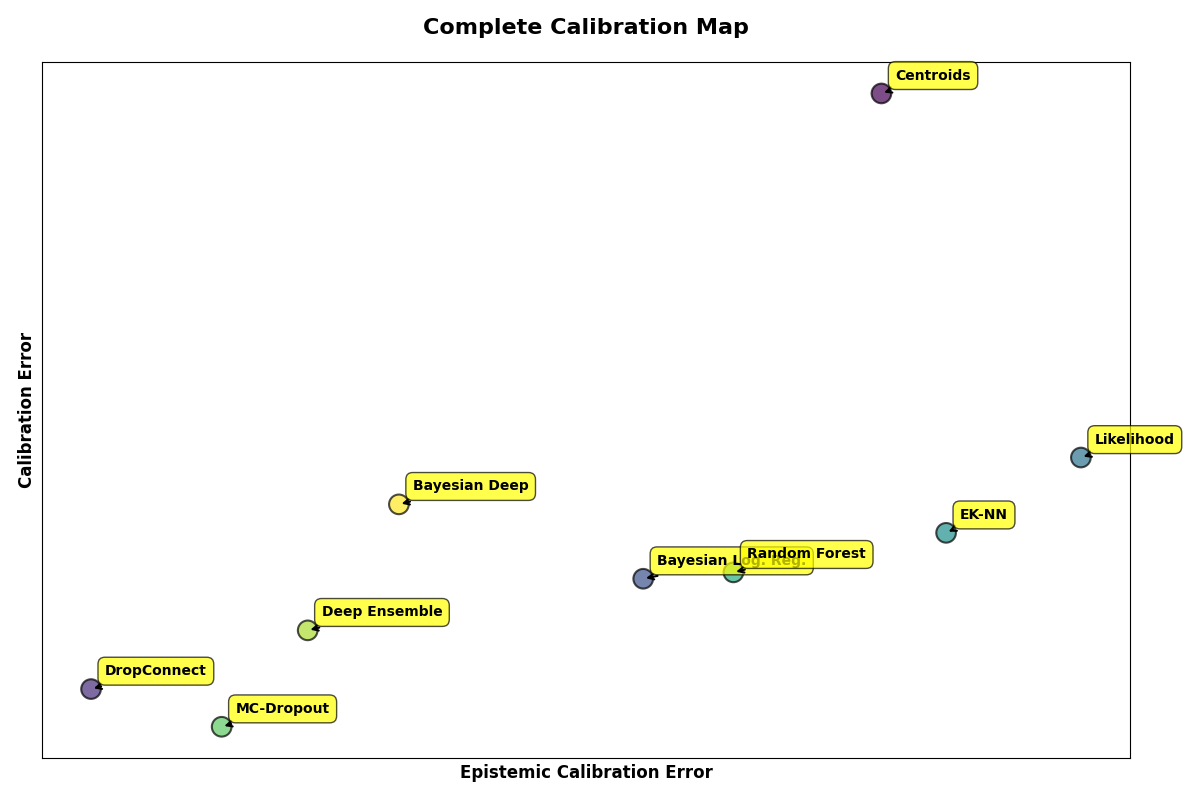}
    \caption{}\label{fig:bench_map}
    \end{subfigure}
    \caption{\ref{fig:bench}: ECE and EECE for multiple uncerainty quantification methods on CIFAR-10 under aggregated One \emph{vs} All classification. \ref{fig:bench_map}: Global calibration map on CIFAR-10, best models are in the bottom-left corner (log-scale).}
\end{figure}

\subsection{Which methods are epistemically calibrated?}\label{exp:best}

In this section, we consider two higher-dimensional vision datasets, MNIST~\cite{726791} and CIFAR-10~\cite{Krizhevsky2009LearningML}. To make the experiments compatible with all studied models, we project the data points into a latent space obtained through an auto-encoder. As our paper only considers the binary classification setting, we divide each problem into 10 \emph{one-vs-all} binary classification tasks. Both datasets comprising 10 classes, the experiment is repeated 10 times, allowing mean and variance to be reported.

On MNIST, the task proved too simple and some models exhibited almost no epistemic calibration error. All results and complementary details are nonetheless available in Appendix~\ref{app:xp}.
Figure~\ref{fig:bench} shows the mean EECE for each model on CIFAR-10. On this dataset, Random Forest remains the worst epistemically calibrated model, followed by Bayesian Logistic Regression, which also exhibits the highest variance. Indeed, while it is a well epistemically calibrated model on linearly separable tasks, it struggles on more complex data, as it is the only model in our study that relies on linear decision boundaries. As expected, on more complex datasets the three illustrative methods still occupy the bottom of the ranking, as they are not designed for a variance-based decomposition of uncertainty.

Figure~\ref{fig:bench} provides an overview of mean results using a log-scale representation: models toward the left are the most epistemically calibrated, while models toward the bottom are the most calibrated in the classical sense. DropConnect and MC-Dropout, which rely on a similar intuition to build diversity, appear overall quite well epistemically calibrated, at least on this task which remains of relative simplicity.

\section{Discussion \& Conclusion}\label{section:conclusion}

In this paper, we introduced the notion of epistemic calibration for second-order models, enabling a finer-grained estimation of the correctness of model predictions than classical calibration in a binary classification setting. We situate our proposal within the existing literature through an impossibility theorem, and reformulate this property with respect to the estimated epistemic uncertainty of the model when considering the variance of the posterior. For a model to be epistemically calibrated, its epistemic uncertainty must match the true expected error of models sampled from the second-order distribution. We further introduced the Expected Epistemic Calibration Error (EECE) and proved that it is a consistent estimator of the True Epistemic Calibration Error (TECE). The EECE is first evaluated on synthetic datasets in order to control the TECE and verify natural behaviour: for instance, models are less epistemically calibrated as noise increases, and better calibrated as training set size grows. We also extend our experiments to deep vision tasks, where DropConnect and MC-Dropout proved to be better epistemically calibrated than their alternatives. A further observation is that epistemic calibration is highly dependent on the inductive bias of the model: for instance, Bayesian Logistic Regression is almost perfectly epistemically calibrated on datasets where its assumptions hold (e.g. linearly separable mixtures of two Gaussian components), while it is poorly epistemically calibrated on complex real-world data.

Ultimately, as machine learning models are increasingly deployed in safety-critical and high-stakes systems, epistemic calibration provides a principled and quantifiable foundation for trust, bridging the gap between uncertainty estimation and the reliability that real-world applications demand.

\bibliographystyle{plain}
\bibliography{ref}

\appendix
\newpage
\setcounter{theorem}{1}
\setcounter{proposition}{0}

\section{Proofs}\label{app:proofs}

\begin{theorem}
    For an epistemically calibrated predictor $\Pi$, the spread of the distribution carries no additional information about $Y$ beyond what the mean already encodes.
\end{theorem}

\begin{proof}
Le $\Pi$ be a second-order model calibrated according to Equation~\eqref{eq:ep-cal} such that
\begin{equation}
    P(Y=1\mid \Pi(X) = \Psi) = \Delta_\Psi,
\end{equation}
where $\Delta(x) = \mathbb{E}_{h\sim\Pi}[h(x)]$ is the mean prediction of the model and $EU(x)$ its variance. As mentioned, fixing a second-order distribution $\Pi$ also naturally fixes its mean $\Delta$, we can thus write
\begin{equation}
\begin{split}
    P(Y=1\mid \Delta(X) = \Delta_\Psi) &= \mathbb{E}[P(Y=1\mid \Pi(X))\mid \Delta(X) = \Delta_\Psi]\\
    &= \Delta_\Psi.
\end{split}
\end{equation}
From this equation we derive the first property induced by epistemic calibration, namely \emph{calibration of the mean}. From this second equation and the epistemic calibration hypothesis, we derive the following property
\begin{equation}
\begin{split}
    P(Y=1\mid \Delta(X) = \Delta_\Psi) &= \Delta_\Psi\\
    &= P(Y=1\mid \Pi(X) = \Psi),
\end{split}
\end{equation}
which we name \emph{epistemic independence}. This can also be reformulated by considering only the variance of the second-order distribution
\begin{equation}
\begin{split}
    P(Y=1\mid \Delta(X) = \Delta_\Psi, EU(X) = EU_\Psi) &= \mathbb{E}[P(Y=1\mid \Pi(X))\mid \Delta(X) = \Delta_\Psi, EU(X) = EU_\Psi]\\
    &= P(Y=1\mid \Delta(X) = \Delta_\Psi).
\end{split}
\end{equation}
This second property guarantees that the variance of the posterior cannot encode information about $Y$ beyond what is already encoded in the mean, i.e. $Y \perp EU(X) \mid\Delta(X)$.
\end{proof}

\begin{proposition}
A second-order predictor $\Pi$ is epistemically calibrated iff, $\forall\, \Psi \in\mathcal{P}([0, 1])$
\begin{equation}
    \mathbb{E}_{\phi \sim \Psi}\!\left[\left(P\!\left(Y=1 \mid \Pi(X) = \Psi\right) - \phi\right)^2\right] = EU_\Psi,
\end{equation}
where $EU_\Psi = \mathrm{Var}_{\phi\sim\Psi}[\phi]$ is the variance of the fixed distribution.
\end{proposition}
\begin{proof}
We want to show that $\mathbb{E}_{\phi \sim \Psi}\!\left[\left(P\!\left(Y=1 \mid \Pi(X) = \Psi\right) - \phi\right)^2\right] = EU_\Psi$ is equivalent to epistemic calibration. Let us first decompose the left-hand side as
\begin{equation}
\begin{split}
     \mathbb{E}_{\phi\sim\Psi}[(P(Y=1\mid \Pi(X) = \Psi) - \phi)^2] &=  \mathbb{E}_{\phi\sim\Psi}[(P(Y=1\mid \Pi(X) = \Psi) - \Delta_\Psi +  \Delta_\Psi -  \phi)^2]\\
     &=  \mathbb{E}_{\phi\sim\Psi}[(P(Y=1\mid \Pi(X) = \Psi) - \Delta_\Psi)^2 +  (\Delta_\Psi -  \phi)^2\\ &\;\;\;+ 2 (P(Y=1\mid \Pi(X) = \Psi) - \Delta_\Psi)(\Delta_\Psi - \phi)]
     \\
     &= (P(Y=1\mid \Pi(X) = \Psi) - \Delta_\Psi)^2 +  EU_\Psi\\ &\;\;\;+ 2(P(Y=1\mid \Pi(X) = \Psi) - \Delta_\Psi)\mathbb{E}_{\phi\sim\Psi}[(\Delta_\Psi -  \phi)] \\
      &= (P(Y=1\mid \Pi(X) = \Psi) - \Delta_\Psi)^2 +  EU_\Psi,
\end{split}
\end{equation}
where $\mathbb{E}_{\phi\sim\Psi}[(\Delta_\Psi -  \phi)] = 0$ by definition. We can further derive
\begin{equation}
\begin{split}
     \mathbb{E}_{\phi\sim\Psi}[(P(Y=1\mid \Pi(X) = \Psi) - \phi)^2] = EU_\Psi \iff P(Y=1\mid \Pi(X) = \Psi) = \Delta_\Psi.
\end{split}
\end{equation}
This final equivalence proves the if-and-only-if statement of the proposition.
\end{proof}

\begin{theorem}
    Let $r(\Psi) := P(Y=1\mid \Pi(X) = \Psi) $ and $\nu(\Psi) := \mathbb{E}_{\phi \sim \Psi}\!\left[\left(P\!\left(Y=1 \mid \Pi(X) = \Psi\right) - \phi\right)^2\right]$ be two Lipschitz continuous calibration functions and $\{I_m\}_{m=1}^M$ be a partition of the joint prediction space $[0,1]^{|H|}$ into convex bins with metric $d$ such that $\underset{M\rightarrow \infty}{\lim} \underset{\;1 < m < M}{\max} \underset{\;a, b \in I_m}{\sup} d(a,b) = 0$. Then
    \begin{equation}
        \underset{M\rightarrow \infty}{\lim}\underset{|H|\rightarrow \infty}{\lim}\underset{N\rightarrow \infty}{\lim}  \mathrm{EECE}_{M,H} = \mathrm{TECE},
    \end{equation}
    with limits in the almost sure sense.
\end{theorem}
\begin{proof}
    With $B_{m}$ the index set for bin $I_{m}$, as defined in section~\ref{section:ep_cal},, we respectively define the the estimators for the calibration function $r$, the average prediction, the proportion of predictions, the average trust and the expected average empirical error
    \[
    \hat{r}_{m} := \frac{1}{|B_{m}|}\sum_{n \in B_{m}}\mathds{1}_{\{y_n = 1\}},\quad\hat{h}_{m} := \frac{1}{|B_{m}|}\sum_{n \in B_{m}} h(x_n),\quad
    \hat{\pi}_{m} := \frac{|B_{m}|}{N},
    \]
    \[
    \hat{\mu}_m := \frac{1}{|B_m|}\sum_{n \in B_m} EU(x_n),\quad
    \hat{\nu}_m := \frac{1}{|H|} \sum_{h\in H} \left(\hat{r}_{m} - \hat{h}_{m}\right)^2.
    \]
    We also denote the expected true distribution, the expected predicted distribution, the proportion of predictions, the expected trust, the expected error and the expected mixture distribution
    \[
    \bar{r}_{m} := P\left(Y = 1 \mid \Pi(X) \in I_m\right),\quad
    \bar{h}_{m} := \mathbb{E}_X [h(X)\mid \Pi(X) \in I_m], \quad\bar{\pi}_m := P(\Pi(X) \in I_m),
    \]
    \[
    \bar{\mu}_m := \mathbb{E}_X [EU(X)\mid \Pi(X) \in I_m],\;\;
    \bar{\nu}_m := \mathbb{E}_X \left[\nu(\Pi(X))\mid \Pi(X) \in I_m\right],\quad \bar{\Pi}_m = \mathbb{E}_X\left[\Pi(X)\mid \Pi(X) \in I_m\right].
    \]
    Let the EECE be rewritten as
    \begin{equation}
    \begin{split}
        \mathrm{EECE}_{M,H} & = \sum_{m=1}^M\frac{|B_m|}{N}\left|\left(\frac{1}{|H|}\sum_{h\in H}\left(\mathrm{acc}\left(I_m\right) - \mathrm{conf}_h\left(I_m\right)\right)^2\right) - \mathrm{trust}\left(I_m\right)\right|\\
        & = \sum_{m=1}^M \hat{\pi}_m\left|\hat{\nu}_m - \hat{\mu}_m \right|.
    \end{split}
    \end{equation}
    We also write the TECE as 
    \begin{equation}
    \begin{split}
        \mathrm{TECE} & = \mathbb{E}_X\left[\left|\mathbb{E}_{h\sim\Pi}\!\left[\left(P\!\left(Y=1 \mid \Pi(X) \right) - h(X)\right)^2\right] - EU(X)\right|\right]\\
        & = \mathbb{E}_X \left[\left|\nu\left(\Pi(X)\right) - EU(X)\right|\right].
    \end{split}
    \end{equation}
    From the law of large numbers $\hat{\pi}_m \xrightarrow{a.s.}\bar{\pi}_m$, and $\hat{\mu}_m \xrightarrow{a.s.}\bar{\mu}_m$. 
    From the Lipschitz continuity of $r$, the continuity of the squared Euclidean norm $||.||_2^2$ and the assumption that the bins width is reduced we have
    \begin{equation}
        \underset{|H|\rightarrow \infty}{\lim}\underset{N\rightarrow \infty}{\lim} \hat{\nu}_m = \bar{\nu}_m,
    \end{equation}
    which is an extended result of Theorem~\ref{theorem:original}, i.e. $\mathrm{ECE}$ is a consistent estimator of $\mathrm{TCE}$ for any function continuous and uniformly continuous in its first argument, as proven by~\cite{Vaicenavicius2019}. Rephrased, $\hat{\nu}_m$ is a consistent estimator of $\bar{\nu}_m$ when sample size an the number of bins increase.
    It follows with the continuous mapping theorem that
    \begin{equation}\label{eq:cmt_ep}
    \begin{split}
        \underset{M\rightarrow \infty}{\lim}\underset{|H|\rightarrow \infty}{\lim}\underset{N\rightarrow \infty}{\lim} \sum_{m=1}^M \left|\hat{\pi}_m| \hat{\nu}_m - \hat{\mu}_m | - \bar{\pi}_m| \nu(\bar{\Pi}_m) - \bar{\mu}_m |\right|
        = \sum_{m=1}^M \bar{\pi}_m\left|| \bar{\nu}_m - \bar{\mu}_m | - | \nu(\bar{\Pi}_m) - \bar{\mu}_m |\right|.
    \end{split}
    \end{equation}
    Let $K \geq 0$ be a Lipschitz constant for the calibration function $\nu$ with metric $d$. We can write
    \begin{equation}\label{eq:lipschitz_ep}
    \begin{split}
        |\bar{\nu}_m - \nu(\bar{\Pi}_m)| &=  |\mathbb{E}_X\left[\nu(\Pi(X)) - \nu(\bar{\Pi}_m) \mid \Pi(X) \in I_m\right]|\\
        & \leq \mathbb{E}_X \left[|\nu(\Pi(X)) - \nu(\bar{\Pi}_m)|\mid \Pi(X) \in I_m\right]\\
        &\leq K\mathbb{E}_X \left[d(\Pi(X), \bar{\Pi}_m)\mid \Pi(X) \in I_m\right]\\
        &\leq K \mathbb{E}_X \left[\underset{a,b\in I_m}{\sup}d(a,b)\mid \Pi(X) \in I_m\right]\\
        &\leq K \underset{\;a, b \in I_{m}}{\sup} d(a,b) \\
        &\leq K \underset{\;1\leq m'\leq M}{\max}\underset{\;a, b \in I_{m'}}{\sup} d(a,b).
    \end{split}
    \end{equation}
    Let $\epsilon > 0$, with $x,y,z\in[0,1]$ and from triangle inequality
    \begin{equation}\label{eq:triangle_ep}
        |x-y| < \epsilon \Rightarrow |x - z| - |y - z| < \epsilon.
    \end{equation}
    From equation~\eqref{eq:lipschitz_ep} and the assumption that 
    \[\underset{M\rightarrow \infty}{\lim} \underset{\;1\leq m\leq M}{\max} \underset{\;a, b \in \;I_m}{\sup} d(a,b) = 0,\]
    there exists $M_0 \in \mathbb{N}$ such that $\forall M \geq M_0, \forall I_{m \leq M}, |\bar{\nu}_m - \nu(\bar{\Pi}_m)| < \epsilon.$
    From equation~\eqref{eq:triangle_ep}, we can write
    \begin{equation}
        |\bar{\nu}_m - \bar{\mu}_m| - |\nu(\bar{\Pi}_m) - \bar{\mu}_m| < \epsilon.
    \end{equation}
    From equation~\eqref{eq:cmt_ep}, we obtain $\forall M \geq M_0$
    \begin{equation}
        \underset{|H|\rightarrow \infty}{\lim}\underset{N\rightarrow \infty}{\lim} \sum_{m=1}^M \left|\hat{\pi}_m| \hat{\nu}_m - \hat{\mu}_m | - \bar{\pi}_m| \nu(\bar{\Pi}_m) - \bar{\mu}_m |\right| < \epsilon,
    \end{equation}
    which implies
    \begin{equation}\label{eq:double_lim_ep}
        \underset{M\rightarrow \infty}{\lim}\underset{|H|\rightarrow \infty}{\lim}\underset{N\rightarrow \infty}{\lim} \sum_{m=1}^M \left|\hat{\pi}_m| \hat{\nu}_m - \hat{\mu}_m | - \bar{\pi}_m| \nu(\bar{\Pi}_m) - \bar{\mu}_m |\right| = 0.
    \end{equation}
    With triangle inequality
    \begin{equation}\label{eq:diff_ep}
    \begin{split}
        |\mathrm{TECE} - \underset{|H|\rightarrow \infty}{\lim}\underset{N\rightarrow \infty}{\lim}\mathrm{EECE}_{M,H}| &= \left|\mathbb{E}_X\left[\left|\nu\left(\Pi(X)\right) - EU(X)\right|\right] - \underset{|H|\rightarrow \infty}{\lim}\underset{N\rightarrow \infty}{\lim}\sum_{m=1}^M \hat{\pi}_m\left| \hat{\nu}_m - \hat{\mu}_m \right|\right|\\
        & \leq  \left|\mathbb{E}_X\left[\left|\nu\left(\Pi(X)\right) - EU(X)\right|\right] - \sum_{m=1}^M \bar{\pi}_m| \nu(\bar{\Pi}_m) - \bar{\mu}_m |\right|\\
        & + \left|\underset{|H|\rightarrow \infty}{\lim}\underset{N\rightarrow \infty}{\lim}\sum_{m=1}^M \hat{\pi}_m\left| \hat{\nu}_m - \hat{\mu}_m \right| - \sum_{m=1}^M \bar{\pi}_m| \nu(\bar{\Pi}_m) - \bar{\mu}_m |\right|.
    \end{split}
    \end{equation}
    By definition of the Riemann-Stieltjes integral
    \begin{equation}
        \underset{M\rightarrow \infty}{\lim}\underset{|H|\rightarrow \infty}{\lim}\left|\mathbb{E}_X\left[\left|\nu\left(\Pi(X)\right) - EU(X)\right|\right] - \sum_{m=1}^M \bar{\pi}_m| \nu(\bar{\Pi}_m) - \bar{\mu}_m |\right| = 0.
    \end{equation}
    Together with equations~\eqref{eq:diff_ep} and~\eqref{eq:double_lim_ep} we obtain
    \begin{equation}
    \begin{split}
        |\mathrm{TECE}\; - \underset{M\rightarrow \infty}{\lim}\underset{|H|\rightarrow \infty}{\lim}\underset{N\rightarrow \infty}{\lim}\mathrm{EECE}_{M,H}| & \leq 0\\
        |\mathrm{TECE}\; - \underset{M\rightarrow \infty}{\lim}\underset{|H|\rightarrow \infty}{\lim}\underset{N\rightarrow \infty}{\lim}\mathrm{EECE}_{M,H}| & = 0,
    \end{split}
    \end{equation}
    with limits in the almost sure sense.
\end{proof}

\section{Protocols \& Additional experiments}\label{app:xp}

\subsection{Protocols}

This section details the implemented models along with the protocol for each experiment.

\subsubsection{Models}

We consider a Bayesian Logistic Regression~\cite{polson2013} with weights and biases following a prior $\mathcal{N}(0, 0.5)$. We also consider Bayesian Deep Learning via Laplace Approximation~\cite{daxberger2022}, using two hidden layers of sizes $(100, 20)$. The model is trained with Cross-Entropy Loss using Adam for 20 epochs. Laplace approximation is computed on the last layer with 100 sampled estimators. For both methods, individual models are sampled from the posterior predictive second-order distribution, and the EECE is computed following Equation~\eqref{eq:eece}.
We include Deep Ensemble~\cite{Lakshminarayanan2017} with 10 estimators, varying only the initialization and training sampling order. A Random Forest~\cite{Breiman2001} is trained with 100 decision trees and a minimum of 5 samples per leaf. DropConnect~\cite{Mobiny2021} and MC-Dropout~\cite{gal2016} are each implemented with two hidden layers $(100, 20)$, trained with Cross-Entropy Loss using Adam for 20 epochs. For DropConnect, a \emph{LinearDropConnect} layer with $p=0.3$ and ReLU activation is applied at each hidden layer. For MC-Dropout, dropout is applied on the two last layers with $p=0.3$ and $p=0.2$, respectively. For all ensemble-based methods, individual members are treated as samples from the second-order predictive distribution.

For illustrative purposes, we include three methods that do not rely on second-order distributions or ensemble decompositions. A centroid-based approach~\cite{van-amersfoort20a} (without the neural network component) is used with length-scale $\sigma = \sqrt{10}$. A likelihood-based approach~\cite{Nguyen2022} computes positive-class plausibility using 10 nearest neighbours, where low plausibility for both classes indicates high epistemic uncertainty. A nearest-neighbours approach~\cite{denoeux1995} with 12 neighbours computes epistemic uncertainty via a normalized conjunctive (Dempster's~\cite{Dempster1967}) combination with $\alpha=0.4$ and $\beta=0.5$. For these methods, hypotheses are sampled from $\phi \sim \mathcal{N}(h(x),\, \sqrt{EU(x)})$, using the predicted epistemic uncertainty as variance. This sampling scheme does not reflect the intended semantics of these methods, but serves to highlight that such a mismatch induces miscalibration with respect to our definition.

\subsubsection{Experiments}

The first experiment considers two Gaussians of equal variance $3$, with means $(-2.5, -2.5)$ and $(2.5, 2.5)$ for the positive and negative classes, respectively. Ten data points are sampled for training, while test instances are sampled along a 2D grid, as shown in Section~\ref{section:xp}. The second experiment reuses the same dataset with 50 training points and $40\,000$ samples for the test and calibration phase, noise is introduced by randomly corrupting input values, and the training set size is varied accordingly.

The third experiment employs an auto-encoder (AE) to project inputs into a shared latent space, enabling comparison across all methods regardless of their compatibility with high-dimensional vision inputs. For MNIST, a convolutional AE is used (encoder: Conv2d $1{\to}16$ stride-2 ReLU, Conv2d $16{\to}32$ stride-2 ReLU, Linear $1568{\to}64$ ReLU. decoder: Linear $64{\to}1568$ ReLU, ConvTranspose2d $32{\to}16$ stride-2 ReLU, ConvTranspose2d $16{\to}1$ stride-2 Sigmoid). For CIFAR-10, a wider convolutional AE is used (encoder: Conv2d $3{\to}16$ stride-2 ReLU, Conv2d $16{\to}32$ stride-2 ReLU, Linear $2048{\to}256$ ReLU. decoder: Linear $256{\to}2048$ ReLU, ConvTranspose2d $32{\to}16$ stride-2 ReLU, ConvTranspose2d $16{\to}3$ stride-2 Sigmoid). Both auto-encoders are trained for 5 epochs with MSE loss and the Adam optimizer. Each model is then trained on the projected latent representations, and EECE is estimated with 100 bins using the K-means binning approach described in the following study.

\subsection{Binning study}

In this study, we consider the Toy dataset in Figure~\ref{fig:dataset} for a completely controlled experiment, where $P(X)$ and $P(Y\mid X)$ are known. This allows to compute a true estimation of the TECE without requiring binning.
The data are generated according to a mixture of Gaussians
\begin{equation}
    P(X) = \sum^K_{k=1}\pi_k\cdot\mathcal{N}(X\mid\mu_k, \Sigma_k),
\end{equation}
where $\mathcal{N}(x\mid\mu_k,\Sigma_k)$ denotes the density of a multivariate Gaussian distribution and $\pi_k$ is the \emph{prior} associated with class $k$.

Figure~\ref{fig:binning} shows the difference between the EECE and the TECE (averaged on the calibration set) for a Deep Ensemble when the number of bin increases. The binning strategy to obtain a partition of the product space is a simple K-means strategy with default scikit-learn~\cite{sklearn2011} parameters and the model is trained on $20$ generated points. The number of bins naturally have an impact on the quality of the results. However, with $10'000$ test instances the EECE is already a good estimator of the TECE with around $200$ bins. Of course, reaching more than this number of bins, the partition becomes sparse and more data-points need to be added into the calibration test set.

\begin{figure}
    \centering
    \includegraphics[width=0.7\linewidth]{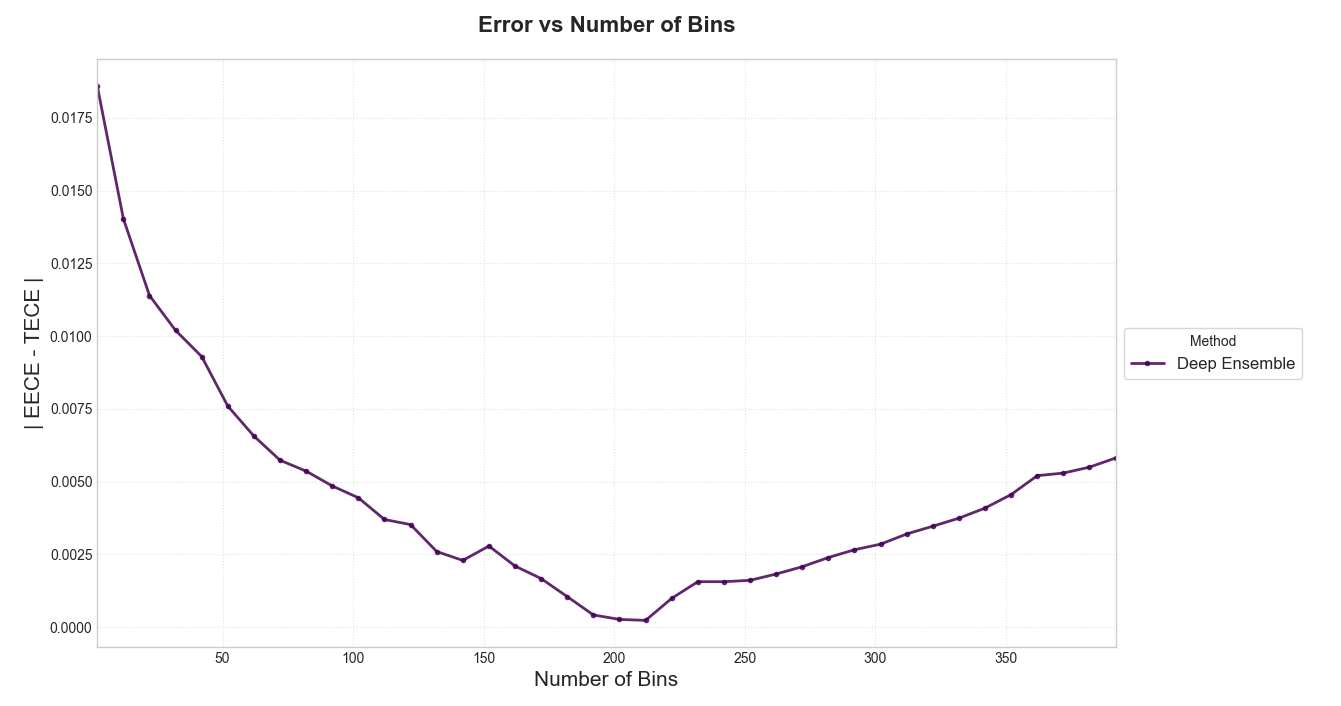}
    \caption{Absolute difference between the estimator through binning and the true expected value \emph{vs} the number of bins with Deep Ensemble. The minimum attained in $\sim$200 bins shows a good compromise between the number of bins and the density of the bins.}\label{fig:binning}
\end{figure}

\begin{figure}
    \centering
    \begin{subfigure}{0.46\linewidth}
    \includegraphics[width=0.48\linewidth]{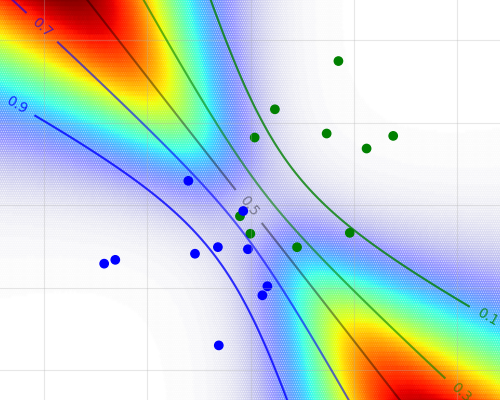}
    \hfill\includegraphics[width=0.48\linewidth]{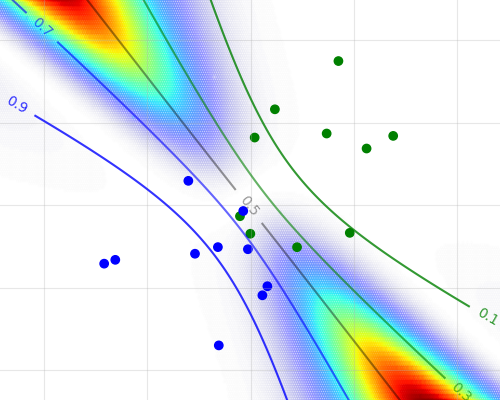}
    \caption{Bayesian Logistic Regression}\label{fig:recal_bayesianLR}
    \end{subfigure}
    \hfill
    \begin{subfigure}{0.46\linewidth}
    \includegraphics[width=0.48\linewidth]{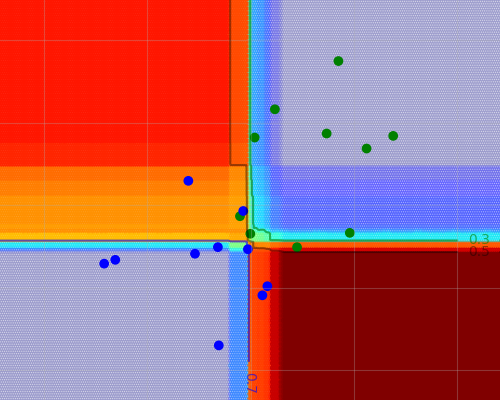}
    \hfill\includegraphics[width=0.48\linewidth]{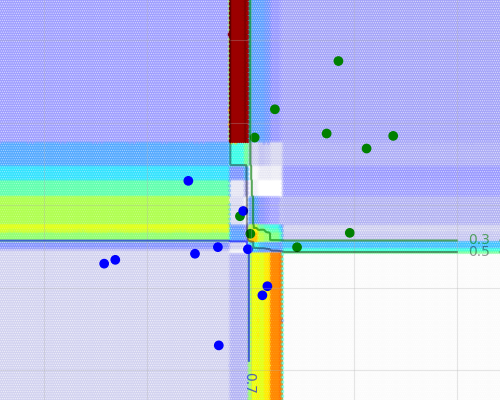}
    \caption{Random Forest}\label{fig:recal_forest}
    \end{subfigure}
    \caption{For each model, the left panel displays the estimated epistemic uncertainty $EU(X)$ and the right panel  displays the expected squared error $\mathbb{E}_{h\sim\Pi}\!\left[\left(P\!\left(Y=1 \mid \Pi(X) \right) - h(X)\right)^2\right]$. Red regions for higher values.}
\end{figure}

\begin{table}
\centering
\scriptsize
\begin{tabular}{cc ccccccccc}
 && \rotatebox{75}{Centroids} & \rotatebox{75}{DropConnect} & \rotatebox{75}{Bayesian Log. Reg.} & \rotatebox{75}{Likelihood} & \rotatebox{75}{EK-NN} & \rotatebox{75}{Random Forest} & \rotatebox{75}{MC-Dropout} & \rotatebox{75}{Deep Ensemble} & \rotatebox{75}{Bayesian Deep} \\
\toprule
\multirow{10}{*}{Acc.} & 0vsAll & 0.94 & 1.00 & 0.99 & 1.00 & 1.00 & 0.99 & 1.00 & 1.00 & 1.00 \\ & 1vsAll & 0.95 & 1.00 & 0.99 & 1.00 & 1.00 & 1.00 & 1.00 & 1.00 & 1.00 \\ & 2vsAll & 0.93 & 0.99 & 0.98 & 1.00 & 1.00 & 0.98 & 1.00 & 1.00 & 1.00 \\ & 3vsAll & 0.88 & 0.99 & 0.97 & 0.99 & 0.99 & 0.98 & 1.00 & 1.00 & 0.99 \\ & 4vsAll & 0.87 & 0.99 & 0.98 & 0.99 & 0.99 & 0.98 & 1.00 & 1.00 & 1.00 \\ & 5vsAll & 0.85 & 0.99 & 0.96 & 1.00 & 1.00 & 0.98 & 1.00 & 1.00 & 1.00 \\ & 6vsAll & 0.95 & 1.00 & 0.98 & 1.00 & 1.00 & 0.99 & 1.00 & 1.00 & 1.00 \\ & 7vsAll & 0.94 & 0.99 & 0.98 & 0.99 & 0.99 & 0.98 & 1.00 & 1.00 & 0.99 \\ & 8vsAll & 0.89 & 0.99 & 0.97 & 0.99 & 0.99 & 0.98 & 0.99 & 1.00 & 0.99 \\ & 9vsAll & 0.83 & 0.99 & 0.96 & 0.99 & 0.99 & 0.98 & 0.99 & 0.99 & 0.99 \\
\midrule
\multirow{10}{*}{ECE} & 0vsAll & 0.35 & 0.01 & 0.01 & 0.01 & 0.03 & 0.04 & 0.01 & 0.01 & 0.01 \\ & 1vsAll & 0.36 & 0.01 & 0.01 & 0.01 & 0.03 & 0.02 & 0.01 & 0.01 & 0.01 \\ & 2vsAll & 0.37 & 0.01 & 0.01 & 0.01 & 0.03 & 0.05 & 0.01 & 0.01 & 0.01 \\ & 3vsAll & 0.36 & 0.02 & 0.02 & 0.01 & 0.03 & 0.05 & 0.01 & 0.01 & 0.01 \\ & 4vsAll & 0.37 & 0.01 & 0.01 & 0.01 & 0.03 & 0.04 & 0.01 & 0.01 & 0.01 \\ & 5vsAll & 0.39 & 0.02 & 0.01 & 0.01 & 0.03 & 0.05 & 0.01 & 0.01 & 0.01 \\ & 6vsAll & 0.37 & 0.01 & 0.01 & 0.01 & 0.03 & 0.03 & 0.01 & 0.01 & 0.01 \\ & 7vsAll & 0.36 & 0.01 & 0.01 & 0.01 & 0.03 & 0.04 & 0.01 & 0.01 & 0.01 \\ & 8vsAll & 0.37 & 0.02 & 0.02 & 0.01 & 0.03 & 0.05 & 0.01 & 0.01 & 0.01 \\ & 9vsAll & 0.37 & 0.01 & 0.02 & 0.01 & 0.03 & 0.04 & 0.01 & 0.01 & 0.01 \\
\midrule
\multirow{10}{*}{EECE} & 0vsAll & 0.07 & 0.00 & 0.00 & 0.14 & 0.04 & 0.01 & 0.00 & 0.00 & 0.00 \\ & 1vsAll & 0.09 & 0.00 & 0.00 & 0.14 & 0.04 & 0.01 & 0.00 & 0.00 & 0.00 \\ & 2vsAll & 0.10 & 0.00 & 0.00 & 0.14 & 0.04 & 0.02 & 0.00 & 0.00 & 0.00 \\ & 3vsAll & 0.09 & 0.00 & 0.00 & 0.14 & 0.04 & 0.02 & 0.00 & 0.00 & 0.00 \\ & 4vsAll & 0.09 & 0.00 & 0.00 & 0.14 & 0.04 & 0.01 & 0.00 & 0.00 & 0.00 \\ & 5vsAll & 0.10 & 0.00 & 0.00 & 0.14 & 0.04 & 0.02 & 0.00 & 0.00 & 0.00 \\ & 6vsAll & 0.09 & 0.00 & 0.00 & 0.14 & 0.04 & 0.01 & 0.00 & 0.00 & 0.00 \\ & 7vsAll & 0.09 & 0.00 & 0.00 & 0.14 & 0.04 & 0.01 & 0.00 & 0.00 & 0.00 \\ & 8vsAll & 0.10 & 0.00 & 0.00 & 0.14 & 0.04 & 0.02 & 0.00 & 0.00 & 0.00 \\ & 9vsAll & 0.09 & 0.00 & 0.00 & 0.14 & 0.04 & 0.02 & 0.00 & 0.00 & 0.00 \\
\bottomrule
\end{tabular}
\caption{Accuracy, ECE and EECE for one-vs-all classification on MNIST.}
\end{table}

\begin{table}
\centering
\scriptsize
\begin{tabular}{cc ccccccccc}
 && \rotatebox{75}{Centroids} & \rotatebox{75}{DropConnect} & \rotatebox{75}{Bayesian Log. Reg.} & \rotatebox{75}{Likelihood} & \rotatebox{75}{EK-NN} & \rotatebox{75}{Random Forest} & \rotatebox{75}{MC-Dropout} & \rotatebox{75}{Deep Ensemble} & \rotatebox{75}{Bayesian Deep} \\
\toprule
\multirow{10}{*}{Acc.} & airplane & 0.71 & 0.71 & 0.93 & 0.93 & 0.91 & 0.91 & 0.92 & 0.92 & 0.92 \\ & auto. & 0.70 & 0.70 & 0.93 & 0.93 & 0.91 & 0.91 & 0.91 & 0.91 & 0.91 \\ & bird & 0.65 & 0.65 & 0.91 & 0.91 & 0.90 & 0.90 & 0.90 & 0.90 & 0.90 \\ & cat & 0.67 & 0.67 & 0.90 & 0.90 & 0.90 & 0.90 & 0.90 & 0.90 & 0.90 \\ & deer & 0.67 & 0.67 & 0.91 & 0.91 & 0.90 & 0.90 & 0.88 & 0.88 & 0.88 \\ & dog & 0.68 & 0.68 & 0.91 & 0.91 & 0.90 & 0.90 & 0.91 & 0.91 & 0.91 \\ & frog & 0.64 & 0.64 & 0.92 & 0.92 & 0.90 & 0.90 & 0.90 & 0.90 & 0.91 \\ & horse & 0.63 & 0.63 & 0.93 & 0.93 & 0.91 & 0.91 & 0.91 & 0.91 & 0.91 \\ & ship & 0.71 & 0.71 & 0.94 & 0.94 & 0.91 & 0.91 & 0.91 & 0.91 & 0.91 \\ & truck & 0.70 & 0.70 & 0.92 & 0.92 & 0.69 & 0.69 & 0.91 & 0.91 & 0.91 \\
\midrule
\multirow{10}{*}{ECE} & airplane & 0.34 & 0.01 & 0.02 & 0.05 & 0.04 & 0.04 & 0.02 & 0.02 & 0.03 \\ & auto. & 0.37 & 0.03 & 0.01 & 0.06 & 0.03 & 0.04 & 0.02 & 0.01 & 0.04 \\ & bird & 0.38 & 0.02 & 0.01 & 0.07 & 0.05 & 0.02 & 0.01 & 0.04 & 0.05 \\ & cat & 0.38 & 0.01 & 0.01 & 0.03 & 0.01 & 0.01 & 0.01 & 0.04 & 0.07 \\ & deer & 0.35 & 0.01 & 0.01 & 0.09 & 0.07 & 0.02 & 0.01 & 0.03 & 0.04 \\ & dog & 0.38 & 0.01 & 0.01 & 0.03 & 0.01 & 0.03 & 0.02 & 0.03 & 0.06 \\ & frog & 0.36 & 0.02 & 0.01 & 0.06 & 0.04 & 0.04 & 0.02 & 0.01 & 0.04 \\ & horse & 0.39 & 0.02 & 0.01 & 0.05 & 0.02 & 0.04 & 0.01 & 0.02 & 0.04 \\ & ship & 0.35 & 0.02 & 0.01 & 0.06 & 0.06 & 0.04 & 0.01 & 0.02 & 0.03 \\ & truck & 0.35 & 0.01 & 0.19 & 0.06 & 0.03 & 0.03 & 0.01 & 0.02 & 0.04 \\
\midrule
\multirow{10}{*}{EECE} & airplane & 0.06 & 0.00 & 0.00 & 0.14 & 0.09 & 0.04 & 0.01 & 0.01 & 0.01 \\ & auto. & 0.07 & 0.00 & 0.00 & 0.15 & 0.08 & 0.03 & 0.01 & 0.01 & 0.01 \\ & bird & 0.07 & 0.00 & 0.00 & 0.14 & 0.09 & 0.04 & 0.00 & 0.01 & 0.01 \\ & cat & 0.07 & 0.00 & 0.00 & 0.15 & 0.09 & 0.04 & 0.00 & 0.01 & 0.01 \\ & deer & 0.07 & 0.00 & 0.00 & 0.13 & 0.09 & 0.04 & 0.01 & 0.01 & 0.01 \\ & dog & 0.07 & 0.00 & 0.00 & 0.15 & 0.09 & 0.04 & 0.01 & 0.01 & 0.01 \\ & frog & 0.07 & 0.00 & 0.00 & 0.14 & 0.09 & 0.04 & 0.01 & 0.01 & 0.01 \\ & horse & 0.08 & 0.00 & 0.00 & 0.15 & 0.08 & 0.04 & 0.01 & 0.01 & 0.01 \\ & ship & 0.05 & 0.00 & 0.00 & 0.14 & 0.08 & 0.04 & 0.01 & 0.00 & 0.01 \\ & truck & 0.06 & 0.00 & 0.24 & 0.15 & 0.08 & 0.04 & 0.01 & 0.01 & 0.01 \\
\bottomrule
\end{tabular}
\caption{Accuracy, ECE and EECE for one-vs-all classification on CIFAR-10.}
\end{table}

\end{document}